\documentclass[a4paper,1p,number,sort&compress,10pt,times,preprint]{article}

\usepackage[includefoot,includehead,a4paper,top=2.5cm,bottom=2.5cm,left=2.5cm,right=2.5cm]{geometry}
\usepackage{times}
\usepackage{graphicx}
\usepackage{authblk}
\usepackage{amsmath,amssymb} 
\usepackage{lineno}
\usepackage{color}
\usepackage{algorithm}
\usepackage{algorithmic}
\usepackage{verbatim}
\usepackage[nodots]{numcompress}
\usepackage{booktabs}
\usepackage{multirow}
\usepackage{rotating}
\usepackage{xspace}
\usepackage{url}
\usepackage{amsthm}
\usepackage{caption}
\usepackage{subcaption}
\usepackage{setspace}
\usepackage[pagebackref=false,breaklinks=true,colorlinks,bookmarks=true,urlcolor=blue,linkcolor=blue,citecolor=blue]{hyperref}
\usepackage{bookmark}
\usepackage{tikz}
\usetikzlibrary{topaths}
\usetikzlibrary{patterns}

\graphicspath{{./}{./figures/}}

\def\Vec#1{{\boldsymbol{#1}}}
\def\Mat#1{{\boldsymbol{#1}}}

\newcommand{\ie}{{i.e.}\@\xspace}

\newcommand{\etal}{{et~al.}\@\xspace}

\newcommand{\tr}{\mbox{tr}}

\newtheorem{theorem}{Theorem}

\begin{document}

\onehalfspacing

\pagestyle{empty}


\begin{center}

{\LARGE\bf Log-Euclidean Bag of Words for Human Action Recognition}

~\\

{
\large
Masoud Faraki~$^{1,2}$, Maziar Palhang~$^{1}$, Conrad Sanderson~$^{3,4}$\\
~\\
$^{1}$~{Isfahan~University~of~Technology, Iran}\\
$^{2}$~{Australian National University, Canberra, Australia}\\
$^{3}$~{University of Queensland, Brisbane, Australia}\\
$^{4}$~{NICTA, Australia}\\
}
\end{center}


\section*{Abstract}

Representing videos by densely extracted local space-time features has recently become a popular approach for analysing actions.
In this paper, we tackle the problem of categorising human actions by devising Bag of Words (BoW) models
based on covariance matrices of spatio-temporal features, with the features formed from histograms of optical flow.
Since covariance matrices form a special type of Riemannian manifold, the space of Symmetric Positive Definite (SPD) matrices,
non-Euclidean geometry should be taken into account while discriminating between covariance matrices.
To this end, we propose to embed SPD manifolds to Euclidean spaces via a diffeomorphism and extend the BoW approach to its Riemannian version.
The proposed BoW approach takes into account the manifold geometry of SPD matrices during the generation of the codebook and histograms.
Experiments on challenging human action datasets show that the proposed method obtains notable improvements in discrimination accuracy,
in comparison to several state-of-the-art methods.
\\


\begin{small}
\noindent
{\bf Published as:}
\begin{itemize}
\item
Masoud Faraki, Maziar Palhang, Conrad Sanderson.\\
Log-Euclidean Bag of Words for Human Action Recognition.\\
IET Computer Vision, Vol.~9, No.~3, pp.~331--339, 2015.\\
\href{http://dx.doi.org/10.1049/iet-cvi.2014.0018}{http://dx.doi.org/10.1049/iet-cvi.2014.0018}
\end{itemize}
\end{small}

\section{Introduction}
\label{sec:intro}

Among several video analysis tasks, human action recognition has received significant attention,
mainly because of its applications to visual surveillance, content-based video analysis, and
human-computer interaction~\cite{aggarwal2011human,carvajal2014,poppe2010survey,reddy_cvprw_2011,weinland2011survey,tsitsoulis2013first}.
Many methods have been proposed for reliable action recognition based on various feature
detectors/descriptors to capture local motion patterns~\cite{dollar2005behavior,laptev2005space,klaser2008spatio,scovanner20073,willems2008efficient,wang2009evaluation,wang2013dense}.
Dense space-time representation of videos has been recently shown to be promising for the action categorisation task~\cite{wang2009evaluation,wang2013dense}.
This in turn suggests the need for employing descriptors to compactly represent the dense collection of local features.

In this paper, we utilise region covariance matrices, composed from densely sampled features, as the descriptors.
Such use of covariance matrices as image descriptors is relatively novel.
They were introduced by Tuzel~\etal~\cite{Cov_Descriptor_ECCV2006} and since then have been employed successfully for pedestrian detection~\cite{Tuzel_2008_PAMI},
non-rigid object tracking~\cite{Porikli:CVPR:2006}, face recognition~\cite{PANG:TCSVT:2008}, and analysing diffusion tensor images~\cite{Pennec_jmiv06}.
Furthermore, a spatio-temporal version of covariance matrix descriptors has shown superior performance for action/gesture recognition~\cite{sanin2013spatio}.

Utilising a covariance matrix as a region descriptor has several advantages.
Firstly, it captures second-order statistics of the local features.
Secondly, it is straightforward approach to fusing various (correlated) features.
Thirdly, it is a low dimensional descriptor and is independent of the size of the region.
Fourth, through the averaging process in its computation, the impact of the noisy samples is reduced.
Finally, efficient methods for its fast computation in images and videos are available~\cite{Tuzel_2008_PAMI,sanin2013spatio}.
While the above advantages make covariance-based descriptors attractive,
using them for discrimination purposes can be challenging.
Covariance matrices are Symmetric Positive Definite (SPD) matrices and naturally form a connected Riemannian manifold.
This can make inference methods based on covariance matrices more difficult, as manifold curvature needs to be taken into account~\cite{Pennec_jmiv06,Tuzel_2008_PAMI}.

Within the fields of image categorisation and face recognition it has been shown that discrimination approaches
based on Bag of Words (BoW) are effective~\cite{nowak_eccv_2006,sanderson_lncs_2009,wiliem_pr_2014,wong_ietbio_2014}.
In a traditional BoW approach, a set of low-level descriptors is typically encoded as a high-dimensional histogram,
with each entry in the histogram representing a count or probability of occurrence of a `visual' word.
The dictionary (or codebook) of words is fixed and obtained during a training stage, typically through {\it k}-means clustering.
The resulting histograms are then interpreted as medium-level feature vectors and fed into standard classifiers.

{\bf Contributions.}
Following the trend of adapting machine learning tools originally designed for vector spaces
to their Riemannian counterparts~\cite{Pennec_jmiv06,Tuzel_2008_PAMI,Subbarao_2009_IJCV,Sra:2011:ECML,Harandi_WACV_2012,harandi2012sparse,sanin2013spatio,shirazi2015,yuan2010human},
in this work we propose to extend the general BoW approach to handle covariance matrices that are treated as points on a Riemannian manifold.
We first form spatio-temporal covariance descriptors from densely extracted motion-based features,
namely Histograms of Optical Flow (HOF) introduced by Laptev \etal~\cite{laptev2008learning}.
The covariance descriptors are then encoded in a Log-Euclidean Bag of Words (LE-BoW) model.
To achieve this, we use a diffeomorphism and form the LE-BoW model by embedding the Riemannian manifold into a vector space.
The embedding is obtained by flattening the manifold through tangent spaces.
We explore several encoding methods within the LE-BoW framework.
We then compare and contrast the proposed approach against recent action recognition methods
proposed by Wang \etal~\cite{wang2013dense}, Messing \etal~\cite{messing2009activity}, and Niebles \etal~\cite{niebles2010modeling}.
Empirical results on three datasets
(KTH~\cite{schuldt2004recognizing}, Olympic Sports~\cite{niebles2010modeling}, Activity of Daily Living~\cite{messing2009activity})
show that the proposed action recognition approach obtains superior performance.

%
%
%
%

We continue this paper as follows.
Section~\ref{sec:related_work} provides an overview of recent work in action recognition.
Section~\ref{sec:preliminary} is dedicated to Riemannian geometry and serves as a grounding for following sections.
Section~\ref{sec:Riemannian_bow} discusses the LE-BoW model.
In Section~\ref{sec:experiments} we compare the performance of the proposed method with previous approaches on several datasets.
The main findings and future directions are summarised in Section~\ref{sec:conclusions}.
\section{Related Work}
\label{sec:related_work}

Human action recognition has been addressed extensively in the computer vision community from various perspectives.
Some methods rely on global descriptors;
two examples are the methods proposed by Ali and Shah~\cite{ali2010human} and Razzaghi \etal~\cite{razzaghi2012new}. 
In~\cite{ali2010human}, a set of optical flow based kinematic features is extracted.
Kinematic models are computed by applying principal component analysis on the volumes of kinematic features.
Razzaghi \etal~\cite{razzaghi2012new} represent human motion by spatio-temporal volume and propose a 
new affine invariant descriptor based on a function of spherical harmonics.
A downside of global representations is their reliance on localisation of the region of interest,
and hence they are sensitive to viewpoint change, noise, and occlusion~\cite{sanin2013spatio}.

To address the abovementioned issues, videos of actions can also be represented through sets of local features,
either in a sparse~\cite{schuldt2004recognizing,dollar2005behavior} or dense~\cite{wang2009evaluation,wang2013dense} manner.
Sparse feature detectors (also referred to as interest point detectors) abstract video 
information by maximising saliency functions at every point in order to extract salient 
spatio-temporal patches. Examples are Harris3D~\cite{Laptev03space-timeinterest} and Cuboid~\cite{dollar2005behavior} detectors.
Laptev and Lindeberg~\cite{Laptev03space-timeinterest} extract interest points at multiple scales using a 3D Harris 
corner detector and subsequently process the extracted points for modelling actions.
The Cuboid detector proposed by Dollar \etal~\cite{dollar2005behavior} extracts salient points based on temporal Gabor filters.
It is especially designed to extract space-time points with local periodic motions.

Wang \etal~\cite{wang2009evaluation} demonstrate that dense sampling approaches consistently 
outperform space-time interest point based methods for human action categorisation.
A dense sampling at regular positions in space and time guarantees good coverage of foreground motions as well as of surrounding context.
To characterise local patterns (\ie motion, appearance, or shape), the descriptors divide small 3D volumes into a grid of $n_x \times n_y \times n_t$ cells
and for each cell the related information is accumulated.
Examples are HOG and HOF~\cite{laptev2008learning}, HOG3D~\cite{klaser2008spatio}, and 3D SIFT~\cite{scovanner20073}.

An alternative line of research proposes to track given spatial point over time and capture related information.
Messing \etal~\cite{messing2009activity} track 
Harris3D~\cite{Laptev03space-timeinterest} interest points with a KLT tracker~\cite{lucas1981iterative} and extract velocity history information.
To improve performance, other useful features such as appearance and location are taken into account in a generative mixture model.
Recently, Wang \etal~\cite{wang2013dense} show promising results by 
tracking densely sampled points and extract aligned shape, appearance, and motion features. They 
also introduce Motion Boundary Histograms (MBH) based on differential optical flow.

\section{Riemannian Geometry}
\label{sec:preliminary}

In this section, we review Riemannian geometry on the manifold of real SPD matrices.
We first formally define a covariance matrix descriptor for the whole video.
Let $\mathbb{I} = \{I_{t}\}_{t = 1}^T$ denote a set of $W \times H$ greyscale frames of a video.
Also, let $\mathbb{O} = \{\Vec{o}_i\}_{i=1}^{n}$ be a set of observations $\Vec{o}_i \in \mathbb{R}^d$ extracted from $\mathbb{I}$.
For example, one might extract a $d$ dimensional feature vector at each pixel, resulting in $W \times H \times T$ observations.
Then, $\mathbb{I}$ can be represented by a $d \times d$ covariance matrix of the observations as:
\begin{eqnarray}
\Mat{C}_{I} & = & \frac{1}{n-1} \sum\nolimits_{i = 1}^{n} \left(\Vec{o}_i - \Vec{\mu} \right)\left(\Vec{o}_i - \Vec{\mu} \right)^T\;, \\
\Vec{\mu}   & = & \frac{1}{n} \sum\nolimits_{i = 1}^{n} \Vec{o}_i\;. \nonumber
\label{eqn:cov_desc}
\end{eqnarray}

The entries on the diagonal of matrix $\Mat{C}_I$ are the variances of each feature
and the non-diagonal entries are their pairwise correlations (see Fig.~\ref{fig:Diagram} for a conceptual diagram).
There are several reasons as to why covariance matrices are attractive for representing images and videos:
{\bf (i)}~they provide a natural way for fusing various features;
{\bf (ii)}~they can reduce the impact of noisy samples through the averaging operation in its computation;
{\bf (iii)}~a $d \times d$ covariance matrix is usually low-dimensional and independent of the size of the region;
{\bf (iv)}~they can be efficiently computed using integral images/videos~\cite{Tuzel_2008_PAMI,sanin2013spatio};
{\bf (v)}~affine invariant metrics exist to compare covariance matrices~\cite{Pennec_jmiv06}.

A \textit{manifold}, $\mathcal{M}$, is a locally Euclidean topological space.
Locally Euclidean means that each point has some neighbourhood that is homeomorphic (one-to-one, onto, and continuous in both directions) to an open ball in $\mathbb{R}^d$, for some $d$.
The \textit {tangent space} at a point {$\Mat{P}$} on the manifold, $T_{\Mat{P}}{\mathcal{M}}$,
is a vector space that consists of the tangent vectors of all possible curves passing through {$\Mat{P}$}
(see Fig.~\ref{fig:manifold} for an illustration).
On the manifold, a \textit {Riemannian metric} is defined as a continuous collection of dot products on the tangent space $T_{\Mat{P}}\mathcal{M}$ at each {{$\Mat{P}$} $\in \mathcal{M}$}.
The Riemannian metric of the manifold enables us to define geometric notions on the manifold such as lengths and angles.
The \textit {geodesic distance} between two points on the manifold is defined as the length of the shortest curve connecting the two points.

A \textit{Riemannian manifold} $(\mathcal{M},\mathrm{g})$ consists of the analytic manifold
$\mathcal{M}$ and its associated metric $\mathrm{g}_\Mat{P}(.,.):\mathcal{M} \times
\mathcal{M}\rightarrow \mathcal{R}$ that varies smoothly on $T_\Mat{P}\mathcal{M}$. The function
$\mathrm{g}$ has a symmetric, positive definite bi-linear form on each {{$\Mat{p}$} $\in T_\Mat{P}\mathcal{M}$}.
It can be chosen to provide robustness to some geometrical transformations.

Two operators, namely the \textit{exponential map} {$\exp_{\Mat{P}}(\cdot):T_{\Mat{P}}{\mathcal{M}} \rightarrow \mathcal{M}$}
and the \textit{logarithm map} {$\log_{\Mat{P}}(\cdot)=\exp^{-1}_{\Mat{P}}(\cdot):\mathcal{M} \rightarrow T_{\Mat{P}}{\mathcal{M}}$},
are defined over differentiable manifolds to switch between the manifold and tangent space at { $\Mat{P}$}.
The exponential operator maps a tangent vector {$\Delta$} to a point {$\Mat{X}$} on the manifold.
The property of the exponential map ensures that the length of {$\Delta$}
becomes equal to the geodesic distance between {$\Mat{X}$} and {$\Mat{P}$}.
The logarithm map is the inverse of the exponential map and maps a point on the manifold to the tangent space {$T_{\Mat{P}}{\mathcal{M}}$}.
The exponential and logarithm maps vary as point {$\Mat{P}$} moves along the manifold.
We refer interested readers to~\cite{BHATIA_2007,Lui2011} for more detailed treatment on manifolds and related topics.

\begin{figure}[!b]
  \centering
  \includegraphics[width=1\linewidth]{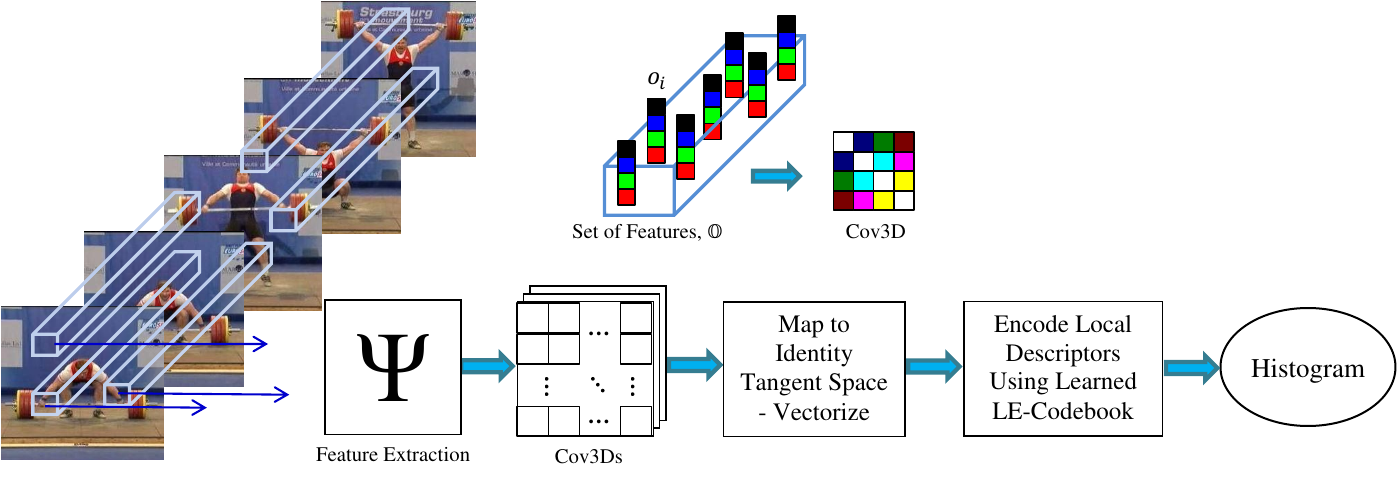}
  \vspace{-2ex}
  \caption
    {
    Conceptual block diagram showing computations of LE-BoW histogram generation.
    }
  \label{fig:Diagram}
\end{figure}

\subsection{Riemannian Manifold of SPD Matrices}

The space of real $d \times d$ SPD matrices, $\mathcal{S}_{++}^d$, forms a Lie Group which is an
algebraic group with a manifold structure. It is natural to use the language of Riemannian manifolds
and all the related concepts of differential geometry when discussing $\mathcal{S}_{++}^d$.

The Affine Invariant Riemannian Metric (AIRM)~\cite{Pennec_jmiv06} on $\mathcal{S}_{++}^d$ is defined as:
\begin{equation}
	\langle \Vec{v}, \Vec{w} \rangle_\Mat{P} :=  \langle \Mat{P}^{-1/2}\Vec{v}\Mat{P}^{-1/2}, \Mat{P}^{-1/2}\Vec{w}\Mat{P}^{-1/2} \rangle
	= \tr \left( \Mat{P}^{-1} \Vec{v} \Mat{P}^{-1} \Vec{w}\right)\;,
	\label{eqn:AIRM_equ}
\end{equation}

\noindent
for $\Mat{P} \in \mathcal{S}_{++}^d$ and $\Vec{v},\Vec{w} \in T_{\Mat{P}}{\mathcal{M}}$,
induces the following geodesic distance between points $\Mat{X},\Mat{Y} \in \mathcal{S}_{++}^d$:
\begin{equation}
\delta_R(\Mat{X},\Mat{Y}) = \|\log(\Mat{X}^{-1/2}\Mat{Y}\Mat{X}^{-1/2})\|_F\;.
\label{eqn:geodesic_distance}
\end{equation}

For the AIRM, the logarithm and exponential maps are given by~\cite{BHATIA_2007}:
\begin{eqnarray}
\log_\Mat{P}(\Mat{X}) & = & \Mat{P}^{\frac{1}{2}}\log(\Mat{P}^{\frac{-1}{2}}\Mat{X}\Mat{P}^{\frac{-1}{2}})\Mat{P}^{\frac{1}{2}}\;, \label{eqn:log_map_AIRM} \\
\exp_\Mat{P}(\Mat{X}) & = & \Mat{P}^{\frac{1}{2}}\exp(\Mat{P}^{\frac{-1}{2}}\Mat{X}\Mat{P}^{\frac{-1}{2}})\Mat{P}^{\frac{1}{2}}\;. \label{eqn:exp_map_AIRM}
\end{eqnarray}

In~Eqns.~(\ref{eqn:log_map_AIRM}) and~(\ref{eqn:exp_map_AIRM}),
$\log(\cdot)$ and $\exp(\cdot)$ are the matrix logarithm and exponential operators, respectively.
For SPD matrices, they can be computed through Singular Value Decomposition (SVD).
If we let $\operatorname{diag} \left( \lambda_1,\lambda_2,\cdots,\lambda_d \right)$ be a diagonal matrix
formed from real values $\lambda_1,\lambda_2,\cdots,\lambda_d$ on diagonal elements
and \mbox{$\Mat{X} = \Mat{U} \operatorname{diag}\left(\lambda_i \right) \Mat{U}^T $} be the SVD of the symmetric matrix \mbox{$\Mat{X}$},
then
\begin{eqnarray}
\log(\Mat{X}) & = & \sum\limits_{r=1}^{\infty}{\frac{(-1)^{r-1}}{r}\left(\Mat{X} - \Mat{I} \right)^r} = \Mat{U} \operatorname{diag}\left(\ln(\lambda_i) \right) \Mat{U}^T, \label{eqn:matrix_log} \\
\exp(\Mat{X}) & = & \sum\limits_{r=0}^{\infty}{\frac{1}{r!}\Mat{X}^r} = \Mat{U} \operatorname{diag}\left(\exp(\lambda_i) \right) \Mat{U}^T. \label{eqn:matrix_exp}
\end{eqnarray}

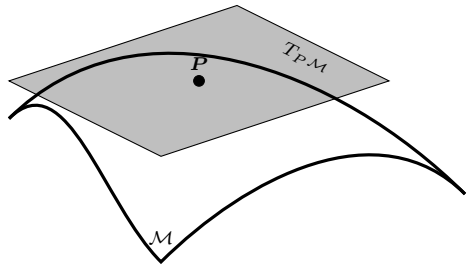
\begin{figure}[!b]
\begin{tikzpicture}
\centering
\filldraw [fill=gray!50,draw=black] (3,2.5) -- (6,3.5) -- node [below,sloped]{$\scriptstyle{T_{P}\mathcal{\scriptscriptstyle M}}$}(8,2.5) -- (5,1.5) -- cycle;
\draw[very thick] (3,2) to[curve to] (9,1);
\draw[very thick] (3,2) to[curve to] (5,0.1) node[above=3pt]{$\mathcal{\scriptstyle {M}}$} to[curve to] (9,1) ;
\filldraw [black] (5.5,2.5) circle (2pt) node[above=0.5pt] {$\Mat{\scriptstyle {P}}$};
\end{tikzpicture}
\centering
\caption{Conceptual illustration of the tangent space at point $\Mat{P}$ on a Riemannian manifold $\mathcal{M}$.}
\label{fig:manifold}
\end{figure}
\section{Log-Euclidean Bag of Words}
\label{sec:Riemannian_bow}

In this section we discuss how a conventional Bag of Words (BoW) model can be extended to incorporate the Riemannian structure of covariance matrices.
In a nutshell, the BoW representation is obtained by first clustering a large set of selected local descriptors with (usually) {\it k}-means,
in order to acquire a visual vocabulary or codebook.
Then, a histogram is extracted by assigning each descriptor to its closest visual word.

To devise a BoW model on Riemannian manifolds, we should address two sub-problems:

\begin{enumerate}

\item
Given a set of training samples {$\mathbb{X} = \left\{  \Mat{X}_i \right\}_{i=1}^{N}$} from the underlying $\mathcal{S}_{++}^d$ manifold
(where each point on the manifold corresponds to a covariance matrix),
how can a codebook $\mathbb{D} = \{ \Mat{D}_j \}_{j=1}^{k}$ be obtained?

\item
Given a codebook $\mathbb{D} = \{  \Mat{D}_j \}_{j=1}^{k}$
and a set of covariance matrices $\mathbb{Q} = \{\Mat{Q}_i\}_{i = 1}^p$ extracted from a query video,
how can a histogram be obtained for classification?

\end{enumerate}

\subsection{Riemannian Codebook}

In the most straightforward case, one can neglect the geometry of SPD matrices and vectorise training data to learn a codebook.
We note that SPD matrices form a closed set under normal matrix addition,
\ie, adding two SPD matrices results in another SPD matrix.
Therefore, a codebook can be generated by applying {\it k}-means on vectorised data.
More specifically, the resulting clusters are determined by computing the arithmetic mean of the nearest training vectors to that cluster.

Despite its simplicity, several studies argue against exploiting Euclidean geometry and vector form of SPD matrices for inference~\cite{Pennec_jmiv06,Tuzel_2008_PAMI}.
For instance, as shown by Pennec~\cite{Pennec_jmiv06} the determinant of the weighted mean could become greater than samples' determinants,
an undesirable outcome known as the swelling effect~\cite{arsigny2007}.
Moreover, symmetric matrices with negative or zero eigenvalues are at a finite distance from any SPD matrix in this framework.
In many problems like diffusion tensor MRI, this is not physically acceptable~\cite{arsigny2007,Pennec_jmiv06}.
Therefore, geometry of SPD matrices should be considered in creating the codebook.

To benefit from Riemannian geometry, an alternative is to replace the arithmetic mean with Karcher mean (also referred as Fr\'{a}chet or Riemannian mean)~\cite{Pennec_jmiv06}.
The Karcher mean is the point that minimises the following metric dispersion:
\begin{equation}
\Mat{X}^\ast= \underset{\Mat{X}}{\arg \min} \sum\nolimits_{i=1}^{N} \delta_g^2(\Mat{X}_i,\Mat{X})\;,
\label{eqn:karcher_mean}
\end{equation}

\noindent
where $\delta_g:\mathcal{M} \times \mathcal{M} \rightarrow \mathbb{R}^+$ is the associated geodesic distance function.
The discussion of the existence and uniqueness value of the Karcher mean as well as its computation are given in~\cite{Pennec_jmiv06}.

Computing the Karcher mean requires switching back and forth between a manifold and its tangent spaces.
This is computationally demanding, especially in our application where a large
number of high dimensional training points is available. More precisely, each mapping to a tangent space
can be computed using Cholesky factorisation with $O(d^3)$ for a $d \times d$ covariance matrix.
Therefore, we opt for a faster way of computing a codebook by minimum use of the logarithm map, \ie, Eqn.~(\ref{eqn:log_map_AIRM}).

Our idea here is to simplify the problem by embedding the manifold into a vector space.
For this purpose, we make use of a mapping from $\mathcal{S}_{++}^d$ into the space of symmetric matrices by
the principal matrix logarithm. The motivation comes from the fact that unlike the general case of invertible square matrices,
there always exists a unique, real and symmetric logarithm for any SPD matrix, which can be obtained by principal logarithm.
Moreover, $\log(\cdot)$ on $\mathcal{S}_{++}^d$ is diffeomorphism (a one-to-one, continuous, differentiable mapping
with a continuous, differentiable inverse). Formally,

\begin{theorem}
$\log(\cdot) : \mathcal{S}_{++}^d \rightarrow Sym(d)$ is $\mathit{C}^{\infty}$ and therefore both $\log(\cdot)$ and
its inverse $\exp(\cdot)$ are smooth, \ie, they are diffeomorphisms.
\end{theorem}

\begin{proof}
We refer the reader to~\cite{arsigny2007} for the proof of this theorem.
\end{proof}

Embedding into the space of $d \times d$ symmetric matrices, $Sym(d)$,
through principal logarithm can be also understood as embedding $\mathcal{S}_{++}^d$ into its tangent space at identity matrix.
Since symmetric matrices (or equivalently tangent spaces) form a vector space,
then we can seamlessly employ Euclidean tools (like {\it k-}means to obtain a codebook) to tackle the problem in hand.
Other properties of the induced space, the log-Euclidean space, are studied in~\cite{arsigny2007}.
We note that our idea here can be labelled as an extrinsic approach, \ie, it depends on the embedding Euclidean space.

Given an SPD matrix $\Mat{X}$, its log-Euclidean vector representation, $\Vec{a} \in \mathbb{R}^{m}, {m=\frac{d(d+1)}{2}}$,
is unique and defined as
$\Vec{a} = \operatorname{Vec}\left(\log(\Mat{X})\right)$
where
$\operatorname{Vec}\left(\Mat{B}\right),\;\Mat{B} \in Sym(d)$ is:
\begin{equation}
\operatorname{Vec}\left(\Mat{B}\right) = \Big[b_{1,1}, \;\sqrt{2}b_{1,2}, \;\sqrt{2}b_{1,3}, \;\cdots
\;\sqrt{2}b_{1,d}, \;b_{2,2}, \;\sqrt{2}b_{2,3}, \;\cdots \;b_{d,d}\Big]^T\;.
\end{equation}

Having the training samples mapped to the identity tangent space,
we seek to estimate $k$ clusters $C_1,C_2,\cdots,C_k$ with centers
$\{\Mat{D}_j\}_{j=1}^{k}$ such that the sum of distances over all clusters is minimised.
This can be solved using the conventional {\it k}-means algorithm~\cite{bishop2006pattern}.
The procedure is summarised in Algorithm~\ref{alg:kmeans_pseudo_code}.

\begin{algorithm}[!tb]
\footnotesize
\begin{algorithmic}[1]
\REQUIRE
~\\
\begin{itemize}
\item
training set {$\mathbb{X} \mbox{=} \left\{  \Mat{X}_i \right\}_{i=1}^{N}$} from the underlying $\mathcal{S}_{++}^d$ manifold
\item
$nIter$, the number of iterations
\end{itemize}
\ENSURE
~\\
\begin{itemize}
\item
Visual dictionary \mbox{$\mathbb{D}= \{  \Mat{D}_j \}_{j=1}^{k}, \Mat{D}_j \in \mathbb{R}^{m}$}
\end{itemize}

\STATE Compute {$\mbox{x} \mbox{=} \left\{  \Vec{x}_i \right\}_{i=1}^{N}$}, log-Euclidean representation of $\mathbb{X}$ using $\Vec{x}_i = \operatorname{Vec}(\log(\Mat{X}_i))$.
\STATE Initialise the dictionary
{$\mathbb{D}= \{  \Mat{D}_j \}_{j=1}^{k}$}
by selecting $k$ samples from $\mbox{x}$ randomly.
\FOR{ $t = 1 \to nIter$}
\STATE Assign each point $\Vec{x}_i$ to its nearest cluster in $\mathbb{D}$.
\STATE Compute the average dispersion from cluster centers by $\varepsilon = \frac{1}{N} \sum_{j=1}^{k}\sum_{\Vec{x}_i \in C_j}dist(\Vec{x}_i,\Mat{D}_j)$.
\STATE If $\varepsilon$ is less than a predefined threshold, then break the loop; else recompute cluster centres $\{  \Mat{D}_j \}_{j=1}^{k}$ by $\Mat{D}_j = \frac{1}{|C_j|} \sum_{\Vec{x}_i \in C_j} \Vec{x}_i$.
\ENDFOR
\caption{Log-Euclidean k-means algorithm over $\mathcal{S}_{++}^d$ for learning the visual dictionary}
\label{alg:kmeans_pseudo_code}
\end{algorithmic}
\end{algorithm}

\subsection{Encoding Local Descriptors}

In the previous section, we elaborated on how a codebook for covariance matrices can be obtained.
In this subsection, we elaborate on several encoding methods for a set of local descriptors.
In other words, having a codebook, $\mathbb{D} = \{\Mat{D}_j\}_{j = 1}^k$, at our disposal
(obtained by Algorithm~\ref{alg:kmeans_pseudo_code}),
we seek to group a set of covariance matrices, $\mathbb{Q} = \{\Mat{Q}_i\}_{i = 1}^p$,
extracted from a query video, in order to find a histogram based representation.
Similar to the codebook learning stage, we first compute the log-Euclidean representation of $\mathbb{Q}$ using $\Vec{q}_i = \operatorname{Vec}(\log(\Mat{Q}_i)), \Vec{q}_i \in \mathbb{R}^{m}$.
Fig.~\ref{fig:Diagram} shows a conceptual diagram of our proposed histogram generation approach.

There are several ways of obtaining a histogram based representation, ranging in terms of complexity and amount of spatial and/or temporal information retained.
In this work we evaluate three methods, elucidated in the following subsections:
{\bf (i)}~hard assignment,
{\bf (ii)}~spatio-temporal pyramids,
{\bf (iii)}~sparse coding.

\subsubsection{Hard Assignment (HA)}
\label{sec:encoding_hard_assignment}

In its most straightforward and simplest form, for the set $\{\Vec{q}\}$, a histogram $H$ is obtained by Hard Assignment (HA),
which is related to Vector Quantisation~\cite{wong_ietbio_2014}.
This requires $p \times k$ comparisons.
The $j$-th ($1 \leq j \leq k$) dimension of $H$ is obtained using
$H_{j} = |C_{j}|$,
where $|C_{j}|$ denotes the number of vectors $\Vec{q}_{i}$ ($1 \leq i \leq p$) assigned to the $j$-th cluster.
The vectors are assigned to their closest vocabulary word in the dictionary using Euclidean distance.
The resulting histogram is $\ell_{2}$ normalised via $\widehat{H} = \frac{H}{\left \| H \right \|}_{2}$.

\subsubsection{Spatio-Temporal Pyramids (STP)}
\label{sec:encoding_spatio_temporal_pyramids}

HA encoding loses structure information between the vectors in the set $\{\Vec{q}\}$.
We encode the structure information to our LE-BoW model by incorporating Spatio-Temporal Pyramids
(STP)~\cite{wang2013dense,laptev2008learning}, an extension of spatial pyramids~\cite{lazebnik2006beyond,wiliem_pr_2014}.
For spatial domain we use the entire spatial block,
a subdivision into three horizontal stripes, and a $2 \times 2$ spatial grid.
For the temporal domain we use the entire duration as well as a subdivision into 2 temporal blocks.
For each cell of the grid, a separate hard assigned LE-BoW histogram is computed.
Then, a video is represented as concatenation of the cell histograms.
We use each grid structure as a separate channel and combine them using a $\chi^{2}$ kernel (see Section~\ref{sec:experiments}).
As illustrated in Fig.~\ref{fig:STP}, we have six channels to represent a video in STP encoding.

\begin{figure}[!tb]
  \centering
  \includegraphics[width=0.8\linewidth, height=1.8in]{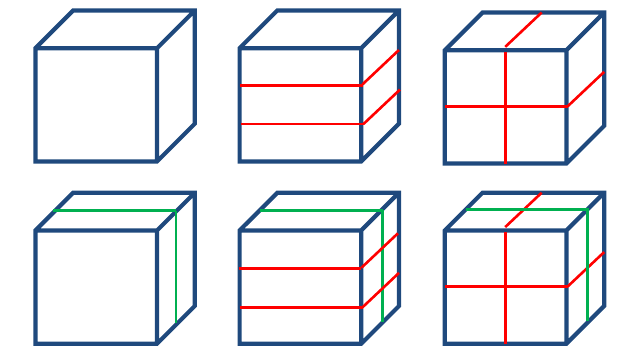}
  \vspace{-2ex}
  \caption
    {
    Spatio-temporal grids for STP histogram encoding.
    }
  \label{fig:STP}
\end{figure}

For classification, we use a non-linear support vector machine with a multi-channel RBF-$\chi ^2$ kernel defined by:
\begin{equation}
K(H_i,H_j) = \exp \left( -\sum_{c}\frac{1}{A^{c}}\delta_{\chi^{2}}(H_{i}^{c},H_{j}^{c}) \right),
\label{eqn:chi2_kernel}
\end{equation}

\noindent
where $\delta_{\chi^{2}}(H_{i}^{c},H_{j}^{c})$ is the $\chi^2$ distance between histogram $H_{i}$ and $H_{j}$
with respect to the $c$-th channel,
and $A^c$ is the mean value of the $\chi^2$ distances between  the training samples for the $c$-th channel.

\subsubsection{Sparse Coding (SC)}
\label{sec:encoding_sparse_coding}

Sparse Coding (SC), the optimal linear decomposition of a signal using a few elements of a
dictionary has proved to be effective for various computer vision tasks~\cite{elad2010sparse,wright2009robust,wong_ietbio_2014}.
Since the resulting histogram by either HA or STP is naturally sparse, it is possible to employ SC algorithms to encode local descriptors.
We use Algorithm~\ref{alg:kmeans_pseudo_code} to train a dictionary for SC.
However, it is also possible to use dedicated algorithms for this purpose~\cite{aharon2006svd,kreutz2003dictionary}.

Kernel sparse coding was previously proposed in~\cite{harandi2012sparse} to take into account the geometry of SPD matrices with the aid of the Stein kernel~\cite{Cherian:PAMI}.
However, the Stein metric fails in our application where many low rank SPD matrices exist.
More specifically, the determinant of SPD matrices formed from HOF features can be close to zero.
As a result, other SPD matrices locate at infinite distance to those low rank matrices.

A vector of weights $\Mat{\alpha} = [ \alpha_{1},\alpha_{2}, \cdots ,\alpha_{k} ]^{T}$ is computed
for each $\Vec{q}_i \in \mathbb{R}^{m}, 1\leq i \leq p$, by solving a minimisation problem that selects a sparse set of dictionary atoms.
More specifically, having a dictionary $\Mat{D} \in \mathbb{R}^{m \times k}$ at our disposal,
the weight vector $\Vec{\alpha} \in \mathbb{R}^{k}$ is obtained via solving the
following $\ell_{1}$-minimisation (also known as Lasso~\cite{elad2010sparse}):
\begin{equation}
\min \frac{1}{2}\left \| \Mat{D} \Vec{\alpha} - \Vec{q}_{i} \right \|_{2}^{2} + \lambda \left \| \Vec{\alpha} \right \|_{1}.
\label{eqn:Sparse_Coding}
\end{equation}

Pooling local sparse codes is performed via averaging.
To solve Eqn.~(\ref{eqn:Sparse_Coding}) we used the SPAMS optimisation toolbox ({\it http://spams-devel.gforge.inria.fr}) for sparse estimation problems.

\subsection{Computational Complexity}

The covariance matrices can be computed efficiently (\ie, in one pass over the video) via integral videos~\cite{sanin2013spatio}.
This results in  $O( WHT d^2)$ operations for computing a $d \times d$ covariance matrix from a $W \times H \times T$ video.
Generating histograms in LE-BoW method requires covariance matrices to be mapped to log-Euclidean space first.
The matrix logarithm can be computed using Cholesky factorisation with $O(d^3)$ operations.
Computing $K$ distances using Euclidean distance can be done at the cost of $O(\frac{1}{2} d^2)$.
Therefore, computing a $K$ dimensional LE-BoW (with HA encoding) signature for one covariance matrix requires $O((WHT+ \frac{K}{2}) d^2 + d^3)$ operations.

\newpage
\section{Experiments}
\label{sec:experiments}

In this section we compare and contrast the performance of the proposed LE-BoW method against several state-of-the-art approaches.
Before delving into experiments, we elaborate how a descriptive representation of action videos can be attained by covariance matrices.
To this end, from each video a set of covariance matrices is extracted and then passed to LE-BoW to generate histograms (see Fig.~\ref{fig:Diagram}).

To generate covariance matrices, a set of overlapping spatio-temporal blocks are extracted from the image sequence
and the covariance matrix for each block is obtained from Histogram of Optical Flow (HOF) features of densely extracted trajectories within that block.
To obtain trajectories, images of a sequence are first resized to $240 \times 360$ and then pixels of an image sequence are sampled on a $W \times W$ spaced grid.
Then, the location of the sampled points is estimated/tracked in $L$ subsequent frames using the estimated optical flow field of the sequence, $\omega$,
convolved with a $3 \times 3$ median filter kernel $M$.
More specifically, given a sampled point $P_t$ in frame $I_t$,
its tracked point $P_{t+1}$ in frame $I_{t+1}$ is obtained via $P_t + M \ast \omega$.
Once the trajectory points $(P_t, P_{t+1}, \cdots, P_{t+L-1})$ in $L$ subsequent frames is found,
the HOF is computed in an $N \times N$ pixels volume around each $P_t$.
To embed structure information, each volume is further divided into a spatio-temporal grid of size $n_\sigma \times n_\sigma \times n_\tau$.

We have used the code available by Wang \etal~\cite{wang2013dense} for our dense trajectory feature
extraction and followed the default parameter values there (\ie $W=5, L=15,$ and $N=32$).
Trajectories are extracted in 8 spatial scales with $n_\sigma = n_\tau = 2$.
Since each HOF is 72 dimensional (\ie 9 bins in $n_\sigma \times n_\sigma \times n_\tau$ grid),
our covariance matrices are 72 $\times$ 72 dimensional.
To avoid having rank deficient covariance matrices, blocks with the number of trajectories below 72 are rejected.
We cluster a subset of 30K randomly selected covariance matrices and fix the number of visual words to 2000.
For classification, we use one-against-all approach and a non-linear support vector machine with a RBF-$\chi^2$ kernel.
We report our LE-BoW model with Hard Assignment (HA), Spatio-Temporal Pyramids (STP), and Sparse Coding (SC) encoding methods.
We show the discrimination power of our proposed method against several state-of-the-art methods previously applied on three datasets:
KTH~\cite{schuldt2004recognizing}, Olympic Sports~\cite{niebles2010modeling}, and Activity of Daily Living~\cite{messing2009activity}.

\def \DTSIZE {0.16}
\begin{figure*}[!b]
  \begin{minipage}{1\textwidth}

  \begin{minipage}{0.04\textwidth}
    \centerline{\bf (a)~}
  \end{minipage}
  \hfill
  \begin{minipage}{0.95\textwidth}
  \includegraphics[width=\DTSIZE \textwidth,keepaspectratio]{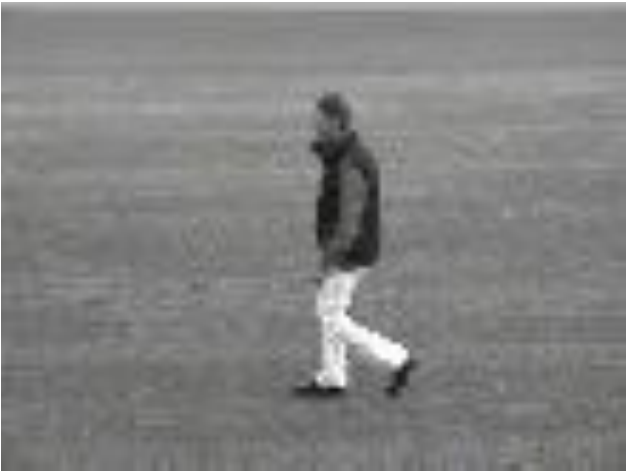}
  \includegraphics[width=\DTSIZE \textwidth,keepaspectratio]{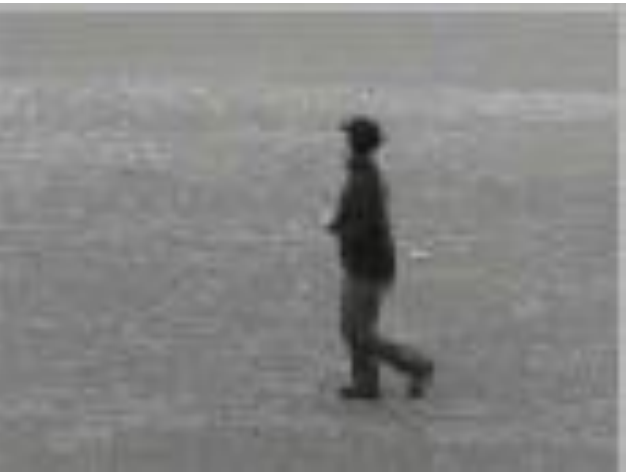}
  \includegraphics[width=\DTSIZE \textwidth,keepaspectratio]{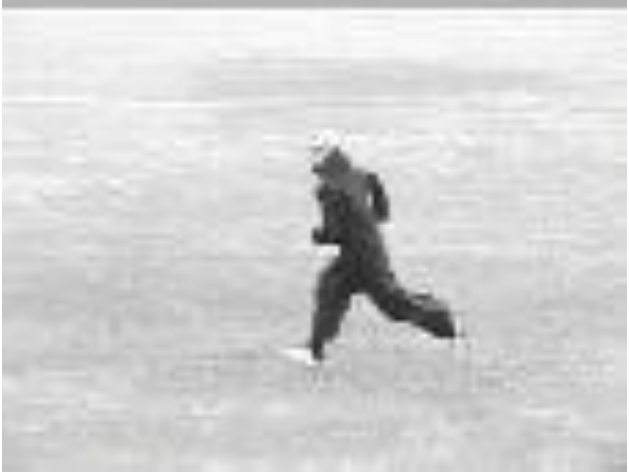}
  \includegraphics[width=\DTSIZE \textwidth,keepaspectratio]{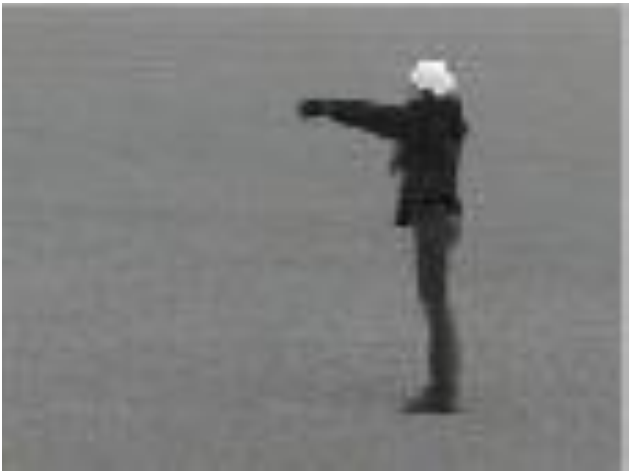}
  \includegraphics[width=\DTSIZE \textwidth,keepaspectratio]{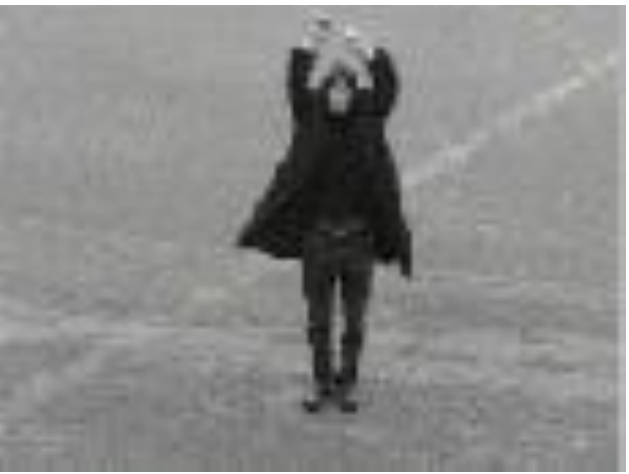}
  \includegraphics[width=\DTSIZE \textwidth,keepaspectratio]{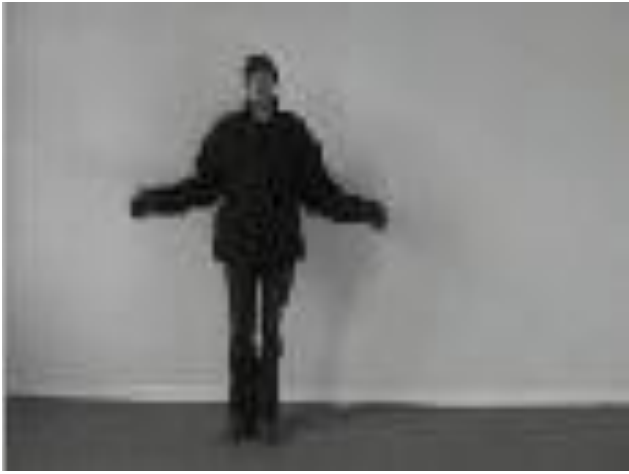}
  \end{minipage}
  
  ~\\
  
  \begin{minipage}{0.04\textwidth}
    \centerline{\bf (b)~}
  \end{minipage}
  \hfill
  \begin{minipage}{0.95\textwidth}
  \includegraphics[width=\DTSIZE \textwidth,keepaspectratio]{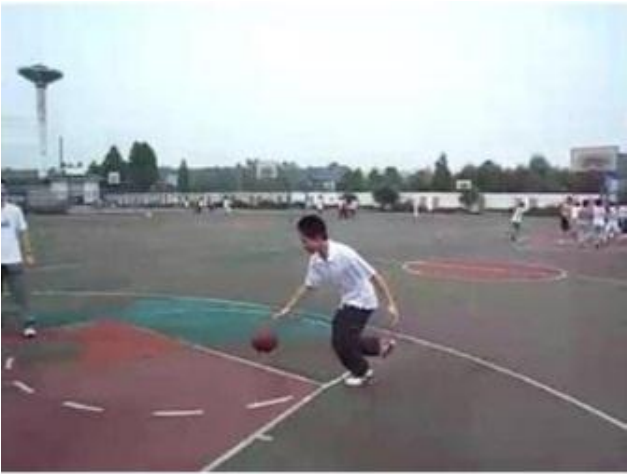}
  \includegraphics[width=\DTSIZE \textwidth,keepaspectratio]{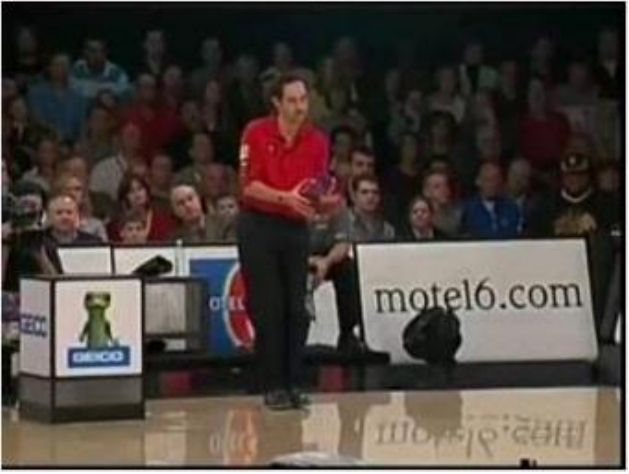}
  \includegraphics[width=\DTSIZE \textwidth,keepaspectratio]{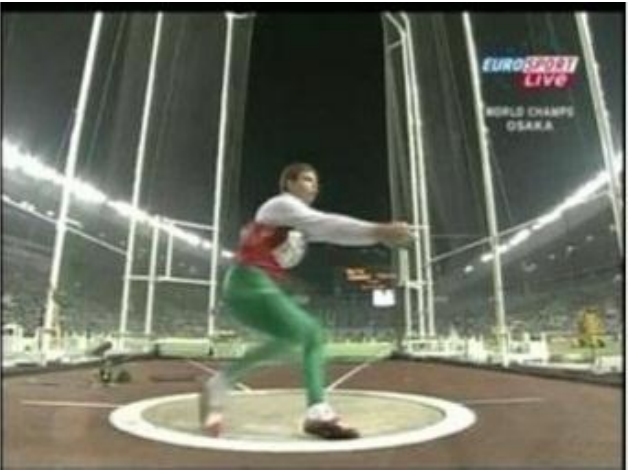}
  \includegraphics[width=\DTSIZE \textwidth,keepaspectratio]{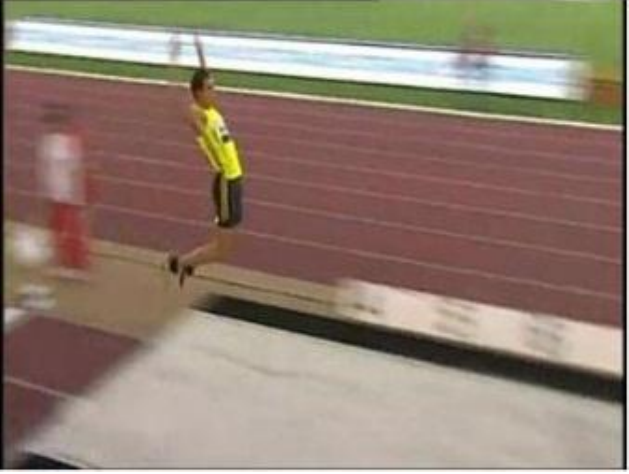}
  \includegraphics[width=\DTSIZE \textwidth,keepaspectratio]{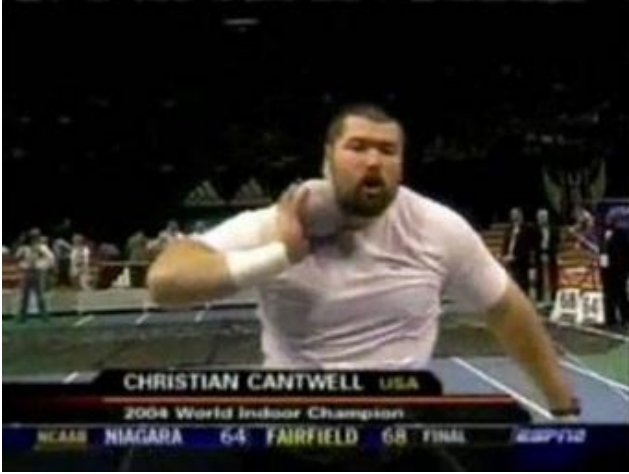}
  \includegraphics[width=\DTSIZE \textwidth,keepaspectratio]{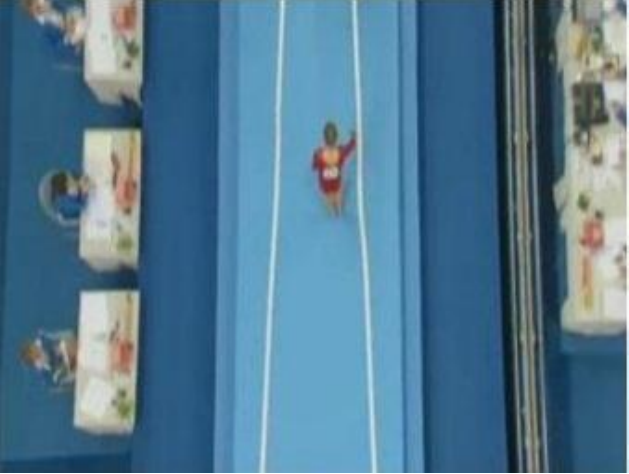}
  \end{minipage}
  
  ~\\
  
  \begin{minipage}{0.04\textwidth}
    \centerline{\bf (c)~}
  \end{minipage}
  \hfill
  \begin{minipage}{0.95\textwidth}
  \includegraphics[width=\DTSIZE \textwidth,keepaspectratio]{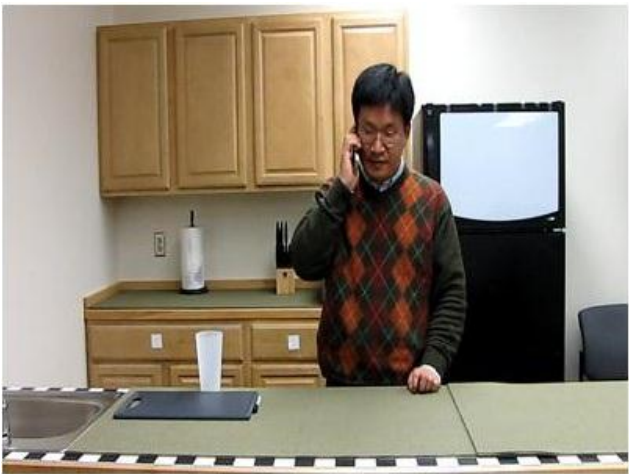}
  \includegraphics[width=\DTSIZE \textwidth,keepaspectratio]{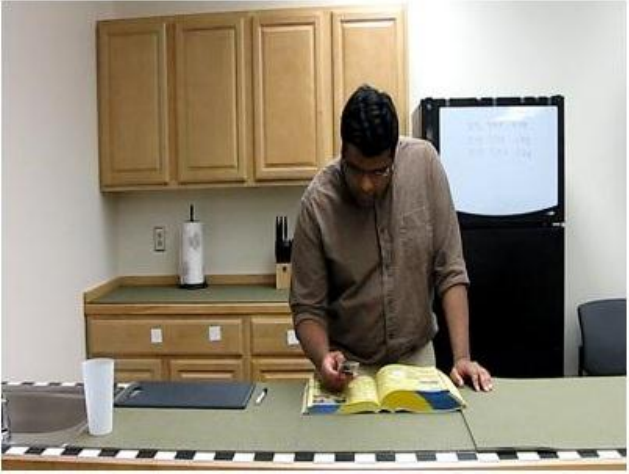}
  \includegraphics[width=\DTSIZE \textwidth,keepaspectratio]{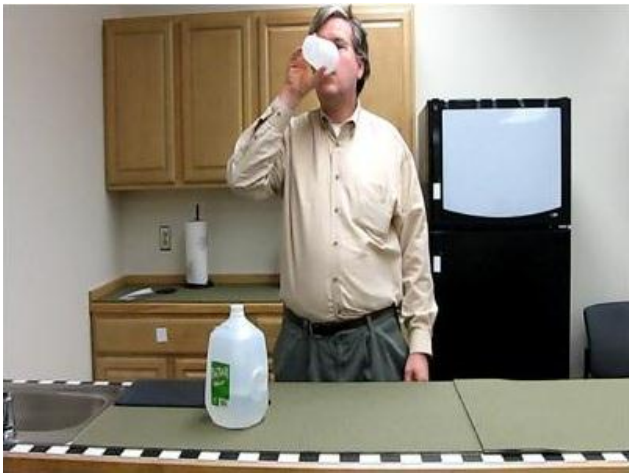}
  \includegraphics[width=\DTSIZE \textwidth,keepaspectratio]{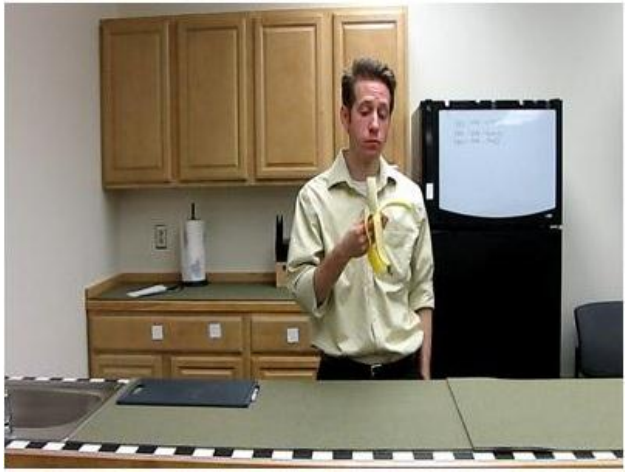}
  \includegraphics[width=\DTSIZE \textwidth,keepaspectratio]{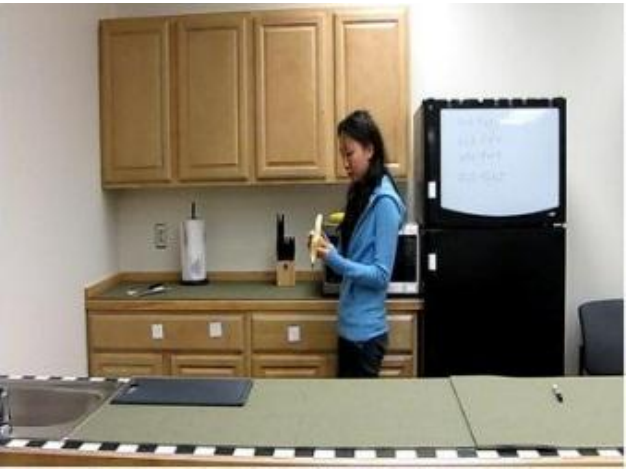}
  \includegraphics[width=\DTSIZE \textwidth,keepaspectratio]{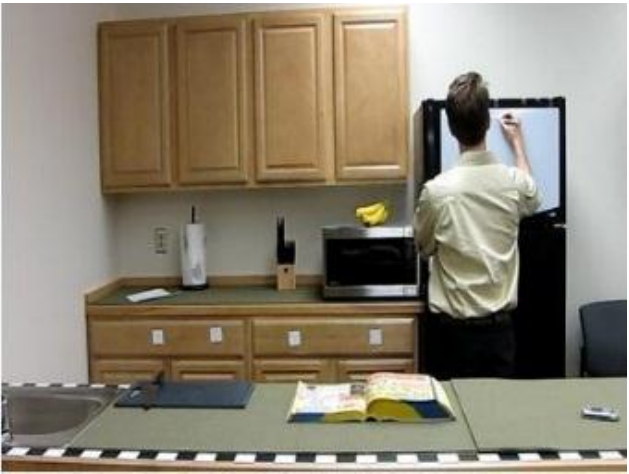}
  \end{minipage}

  \end{minipage}
  \caption
    {
    Example images from the datasets used in our experiments:
    {\bf (a)}~KTH~\cite{schuldt2004recognizing},
    {\bf (b)}~Olympic Sports~\cite{niebles2010modeling},
    {\bf (c)}~Activity of Daily Living~\cite{messing2009activity}.
    }
  \label{fig:action_example}
\end{figure*}

\subsection{KTH Dataset}

The KTH dataset~\cite{schuldt2004recognizing} contains six human action classes:
walking, jogging, running, boxing, hand-waving, and hand-clapping,
performed by 25 subjects in 4 scenarios:
outdoors, outdoors with scale variation, outdoors with varying clothes,
and indoors; see Fig.~\ref{fig:action_example} for examples.
The videos are recorded with static and homogeneous background.
However, the camera is not static, \ie vibration and unknown zooming exist.
In total, the data consists of 2391 video samples.
We follow the original experiment setup of the authors
(\ie, dividing the samples into subjects: 2, 3, 5, 6, 7, 8, 9, 10, 22 for the test set and the remaining 16 subjects for the training set).

On KTH dataset, Laptev \etal~\cite{laptev2008learning} proposed a system where spatio-temporal interest points are extracted and described using HOG/HOF descriptors.
In order to classify a query video, BoW model is utilised in a multi-channel SVM classifier with $\chi^2$ kernel.
Gilbert \etal~\cite{gilbert2011action} propose to use an overcomplete set of simple 2D corners in space and time.
The extracted points are first grouped spatially and temporally using a hierarchical process.
The most distinctive and descriptive features are learned.
Wang \etal~\cite{wang2013dense} track densely sampled points by a median filter kernel and extract aligned shape, appearance, and motion features.
BoW model is utilised in a 30-channel (6 channels and 5 types of features) SVM classifier with $\chi^2$ kernel for classification.

In Table~\ref{tab:LE-BoW_Others_KTH}, we compare our proposed method against the aforementioned methods on the KTH dataset.
Our LE-BoW approach is superior to the method proposed by Laptev \etal and Gilbert \etal.
The HA and STP encoding methods result in performance on par with Wang's system.
However, the SC encoding method improved the accuracy by approximately 2 percentage points.

\begin{table*}[!tb]
  \centering
 \caption
    {
       Comparison between the proposed approach against previous methods on the KTH dataset;
       CCR: Correct	Classification Rate (in \%).
    }
   \label{tab:LE-BoW_Others_KTH}
  \begin{tabular}{lcc}
    \toprule
    {\bf Method}   & {\bf ~~CCR~~} \\
    \toprule
    {Laptev \etal~\cite{laptev2008learning}}       & 91.8       \\
    {Gilbert \etal~\cite{gilbert2011action}}   & 94.5       \\
    {Wang \etal~\cite{wang2013dense}}           & 95.3       \\
    {\bf proposed LE-BoW (HA)}       & 95.0		  \\
    {\bf proposed LE-BoW (STP)}           & { 95.7}  \\
    {\bf proposed LE-BoW (SC)}           & {\bf 97.4}  \\
    \bottomrule
  \end{tabular}
\end{table*}

\subsection{Olympic Sports Dataset}

The Olympic Sports dataset~\cite{niebles2010modeling} contains videos of athletes practising various sport activities.
All video sequences were collected from YouTube and their class labels annotated with the help of Amazon Mechanical Turk.
There are 16 sports actions: high-jump, long-jump, triple-jump, pole-vault, basketball lay-up, bowling, tennis-serve, platform, discus,
hammer, javelin, shot-put, springboard, snatch, clean-jerk, and vault, represented by a total of 783 videos.
We use the standard train/test split recommended by the authors (649 sequences for training and 134 sequences for testing).
Example images are shown in Fig.~\ref{fig:action_example}.

Niebles \etal~\cite{niebles2010modeling} represent activities as temporal compositions of motion segments.
They train a discriminative model that encodes a temporal decomposition of video sequences and appearance models for each motion segment.
For classification, a query video is matched to the model according to the learned appearances and motion segment decomposition.
The classification is based on the quality of matching between the motion segment classifiers and the temporal segments in the query sequence.
In Liu \etal~\cite{liu2011recognizing}, human actions are represented by a set of action attributes.
A unified framework introduced wherein the attributes can be discriminatively selected.
The framework is built upon a latent SVM formulation where latent variables capture the degree of importance of each attribute for each action class.

In Table~\ref{tab:LE-BoW_Others_OS}, we compare our proposed method against state-of-the-art methods on the Olympic Sports dataset.
Mean Average Precision over all classes is reported as in~\cite{niebles2010modeling}.
Using spatio-temporal pyramids consistently improved the classification rate (from 74.9\% to 80.6\%).
The same improvement is also observed in the work of Wang \etal~\cite{wang2013dense} by considering trajectory shape, HOG, HOF, MBHx, and MBHy descriptors in a STP encoding approach.
However, Wang \etal report 58.7\% by single HOF descriptor.
In contrast, we observed considerable improvement by taking the covariance of HOF features (\ie 74.9\% by HA).
On this dataset, the SC encoding method works on par with the STP encoding approach.

\begin{table*}[!tb]
  \centering
 \caption
    {
       Comparisons between the proposed approach to the state-of-the-art methods on Olympic Sports  dataset;
       MAP: Mean Average Precision over all classes (in \%).
    }
   \label{tab:LE-BoW_Others_OS}
  \begin{tabular}{lcc}
    \toprule
    {\bf Method}   & {\bf ~~MAP~~} \\
    \toprule
    {Niebles \etal~\cite{niebles2010modeling}}       & 72.1       \\
    {Liu \etal~\cite{liu2011recognizing}}   & 74.4       \\
    {Wang \etal~\cite{wang2013dense}}           & 77.2       \\
    {\bf proposed LE-BoW (HA)}       & 74.9		  \\
    {\bf proposed LE-BoW (STP)}           & {\bf 80.6}  \\
    {\bf proposed LE-BoW (SC)}           & { 79.9}  \\
    \bottomrule
  \end{tabular}
\end{table*}

\subsection{Activity of Daily Living Dataset}

This dataset consists of 150 videos of 5 subjects performing a series of daily tasks in a kitchen environment, acquired using a stationary camera~\cite{messing2009activity}.
As recommended by~\cite{messing2009activity}, we evaluate our results on this dataset using 5-fold cross validation.
In each fold, videos from four individuals are considered for training and the fifth for testing.
Sample frames are shown in Fig.~\ref{fig:action_example}.

We compare the proposed LE-BoW approach against 3 state-of-the-art human action classification systems:
\textbf{(i)}~Laptev \etal~\cite{laptev2008learning},
\textbf{(ii)}~Matikainen \etal~\cite{matikainen2010representing},
\textbf{(iii)}~Messing \etal~\cite{messing2009activity}.
In~\cite{matikainen2010representing}, a method for augmenting quantised local features with relative spatial-temporal relationships between pairs of features is proposed.
Their discriminative classifier is trained by estimating all of the cross probabilities for various local features of an action.
Messing \etal~\cite{messing2009activity} track Harris3D~\cite{Laptev03space-timeinterest} interest points with a KLT tracker~\cite{lucas1981iterative}
and extract velocity history information along the trajectories.
Appearance and location features are utilised in a generative mixture model to improve the recognition performance.

Table~\ref{tab:LE-BoW_Others_ADL} shows that the proposed LE-BoW approach with simple HA encoding outperforms the state-of-the-art methods.
SC encoding obtains the highest performance in which the correct recognition accuracies for individual subjects are: 86.7, 90.0, 93.3, 90.0, 96.7.
In total, SC encoding was unable to correctly classify only 13 videos (out of 150).
STP encoding, with the mean correct classification rate of 90.7\%, is on par with SC encoding.

\begin{table*}[!tb]
  \centering
 \caption
    {
       Comparisons between the proposed approach to the state-of-the-art methods on Activity of Daily Living dataset;
       CCR: Correct	Classification Rate (in \%).
    }
   \label{tab:LE-BoW_Others_ADL}
  \begin{tabular}{lcc}
    \toprule
    {\bf Method}   & {\bf ~~CCR~~} \\
    \toprule
    {Laptev \etal~\cite{laptev2008learning}}       & 80      \\
    {Messing \etal~\cite{messing2009activity}}   & 89       \\
    {Matikainen \etal~\cite{matikainen2010representing}}           & 70       \\
    {\bf proposed LE-BoW (HA)}       & 90.0		  \\
    {\bf proposed LE-BoW (STP)}           & { 90.7}  \\
    {\bf proposed LE-BoW (SC)}           & {\bf 91.3}  \\
    \bottomrule
  \end{tabular}
\vspace{3ex}
\end{table*}

\newpage
\subsection{Trade-Off Between Accuracy and Complexity}

Encoding local descriptors via HA (Section~\ref{sec:encoding_hard_assignment}) has the least computational burden.
However, it loses the structure information between the set of covariance matrices extracted from a query video.
Notable improvements are observed by encoding structure information using
STP (Section~\ref{sec:encoding_spatio_temporal_pyramids})
or
SC (Section~\ref{sec:encoding_sparse_coding}).
Compared to both HA and STP, obtaining the histogram using SC is considerably more computationally demanding,
as SC requires solving a minimisation problem for each spatio-temporal block.

The experiments show that STP provides a good trade-off between the recognition accuracy and computational cost.
For example, by averaging over 10 runs and 100 videos, histogram generation using SC for one video took 33 seconds on a 2.5GHz Intel i5 CPU with 4GB of RAM using Matlab,
while the same test just needed 1.2 and 0.6 seconds for STP and HA, respectively.

\subsection{Further Discussion}

In this subsection we provide further insight into the performance and properties of the proposed method.
We first provide more performance results using Recall and Precision on the studied datasets.
We then asses the effect of various spatio-temporal grids in the HOF descriptor on recognition accuracy.

Tables~\ref{tab:Recall_Prec_KTH} to~\ref{tab:Recall_Prec_ADL}
show the performance of the proposed LE-BoW using STP encoding in terms of Precision and Recall measures.
On the KTH dataset, the ``boxing'' action obtained the highest performance according to both of the measures.
The ``hand clapping'' action obtained the second best Recall value,
while ``hand waving'' obtained the second best Precision value.
On the Olympic Sports dataset, none of the actions were confused with ``javelin throw''.
On the Activity of Daily Living dataset, ``drink water'' and ``lookup in phonebook'' were reliably recognised, \ie, 100\% on both measures.

We study the effect of several spatio-temporal grids in HOF computation on recognition accuracy of KTH actions.
Fig.~\ref{fig:graph} shows the performance for various values of $n_\sigma \times n_\sigma \times n_\tau$.
Starting from the grid $1 \times 1 \times 2$ to $2 \times 2 \times 2$, the classification accuracy improves with further increasing the number of cells.
However, there is a notable drop by moving from the grid $2 \times 2 \times 2$ to $2 \times 2 \times 3$.
This is not surprising because we reject blocks with the trajectory numbers less than the dimension of HOF (\ie 108 in this structure, $2 \times 2 \times 3 \times 9 = 108$).
We observed that many blocks are rejected with the threshold value equal to 108.
As a result, the final histogram is not rich enough (compared to the grid $2 \times 2 \times 2$).
Further increasing in the number of spatial cells, \ie $n_\sigma=3$ does not yield better results and hence we opted for the grid $2 \times 2 \times 2$ in all our experiments.

\begin{table*}[!tb]
\centering
\caption
    {
       Precision and Recall on KTH dataset with STP encoding (values are in \%).
    }
\label{tab:Recall_Prec_KTH}
\begin{tabular}{l*{6}{c}r}
              & Boxing & Hand Clapping & Hand Waving & Jogging & Running  & Walking  \\
\hline
{\bf Precision}         & 98.6 & 94.6 & 97.1 & 95.1 & 93.8 & 95.2   \\
{\bf Recall}            & 97.9 & 97.2 & 95.1 & 93.8 &  94.4 & 95.8   \\
\end{tabular}
\vspace{3ex}
\end{table*}

\begin{table*}[!tb]
\centering
\caption
    {
       Precision and Recall on Olympic Sports dataset with STP encoding (values are in \%).
    }
\label{tab:Recall_Prec_OS}
\begin{tabular}{lcccccc}
               & Basketball Layup & Bowling & Clean Jerk & Discus Throw & Hammer Throw  & High Jump\\
\hline
{\bf Precision}         & 80.0 & 80.0 & 77.8 & 76.9 & 75.0 & 83.3   \\
{\bf Recall}            & 80.0 & 88.9 & 70.0 & 90.9 &  75.0 & 90.9  \\
\end{tabular}

\vspace{2ex}

\begin{tabular}{lccccccc}
                  & Javelin Throw  & Long Jump & Platform 10m & Pole Vault & Shot Put & Snatch & Springboard 3m \\
\hline
Precision         & 100 & 71.4 & 77.8 & 85.7 & 88.9 & 77.8  & 85.7 \\
Recall            & 75.0 & 83.3 & 77.8 & 75.0 &  80.0 & 77.8 &  75.0 \\
\end{tabular}

\vspace{2ex}

\begin{tabular}{lccc}
               & Tennis Serve & Triple Jump & Vault\\
\hline
{\bf Precision}          & 85.7 & 75.0 & 90.0    \\
{\bf Recall}             & 85.7 & 75.0 & 90.0   \\
\end{tabular}
\vspace{3ex}
\end{table*}

\begin{table*}[!tb]
\centering
\caption
    {
       Precision and Recall on Activity of Daily Living dataset with STP encoding (values are in \%).
    }
\label{tab:Recall_Prec_ADL}
\begin{tabular}{lcccccc}
               & Answer Phone & Chop Banana & Dial Phone & Drink Water & Eat Banana & Eat Snack \\
\hline
{\bf Precision}         & 86.7 & 86.7 & 86.7 & 100 & 81.3  & 86.7  \\
{\bf Recall}            & 86.7 & 86.7 & 86.7 & 100 &  86.7  & 86.7 \\
\end{tabular}

\vspace{2ex}

\begin{tabular}{lccccc}
                  & Lookup In Phonebook & Peel Banana & Use Silverware & Write on Whiteboard \\
\hline
{\bf Precision}         & 100 & 92.9 & 93.3 & 93.3  \\
{\bf Recall}            & 100 & 86.7 & 93.3 &  93.3  \\
\end{tabular}
\vspace{3ex}
\end{table*}

\begin{figure}[!tb]
  \centering
  \includegraphics[width=0.7\linewidth]{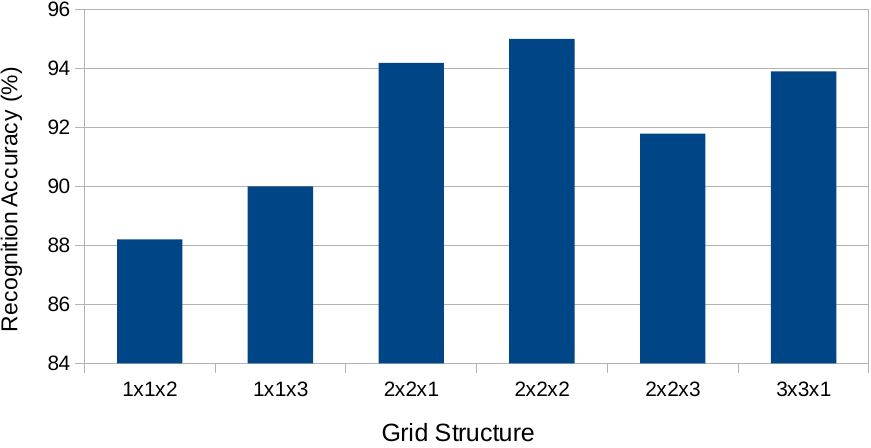}
  \caption
    {
    Evaluation of cell grid structure parameters on KTH dataset.
    }
  \label{fig:graph}
\end{figure}

\section{Conclusions}
\label{sec:conclusions}

We devised an approach to extend the popular Bag of Words (BoW) models to a special class of non-Euclidean spaces,
the space of Symmetric Positive Definite (SPD) matrices formed by 
covariance descriptors of spatio-temporal features~\cite{sanin2013spatio}.
In doing so, we elaborated on how a codebook and subsequently histograms can be obtained for covariance matrices
and devised Log-Euclidean Bag of Words (LE-BoW), an extrinsic extension of conventional BoW using Riemannian geometry of SPD matrices. 
The main ingredient of our proposal is a diffeomorphism that embeds Riemannian manifold of SPD matrices into an Euclidean space.
This is consistent with several studies~\cite{arsigny2007,Caseiro:PR:2012,Hu:PAMI:2012,Guo:TIP:2013}
that demonstrate the benefit of such embedding.

The proposed framework was validated by experiments on three challenging action recognition datasets,
namely KTH~\cite{schuldt2004recognizing}, Olympic Sports~\cite{niebles2010modeling}, and Activity of Daily Living~\cite{messing2009activity}.
The experiments show that the proposed LE-BoW approach for classifying human actions performs better than the state-of-the-art methods
proposed by Laptev \etal~\cite{laptev2008learning}, Niebles \etal~\cite{niebles2010modeling}, and Wang \etal~\cite{wang2013dense}.
We believe that our work motivates future research on extending well-known machine learning inference tools to their Riemannian counterparts.
\section*{Acknowledgements}

NICTA is funded by the Australian Government through the Department of Communications, as well as the Australian Research Council through the ICT Centre of Excellence program.

\bibliographystyle{ieee}

\bibliography{references}

\end{document}